\documentclass[11pt]{article}

\usepackage{authblk}
\author{Ingvar Ziemann}

\affil{University of Pennsylvania}

\usepackage[utf8]{inputenc} 
\usepackage[T1]{fontenc} 

\usepackage{fullpage}

\usepackage{url}            
\usepackage{booktabs}       
\usepackage{amsfonts}       
\usepackage{nicefrac}       
\usepackage{microtype}      
\usepackage[dvipsnames]{xcolor}

\usepackage{natbib} 
\usepackage{enumitem}
\usepackage{amsthm}
\usepackage{amssymb}
\usepackage{amsmath}
\usepackage{mathrsfs}
\usepackage{thmtools}
\usepackage{hyperref}
\hypersetup{
    pdftoolbar=true,        
    pdfmenubar=true,        
    pdffitwindow=false,     
    pdfstartview={FitH},    
    pdfnewwindow=true,      
    colorlinks=true,       
    linkcolor=blue,          
    citecolor=ForestGreen,        
    filecolor=magenta,      
    urlcolor=red,           
    breaklinks=false,
}

\usepackage{cleveref}
\usepackage{thmtools}
\usepackage{thm-restate}

\numberwithin{equation}{section}

\date{}

\title{A Short Information-Theoretic Analysis of Linear Auto-Regressive Learning}

\newcommand{\E}{\mathbf{E}}

\DeclareMathOperator{\tr}{tr}

\newcommand{\T}{\mathsf{T}}

\newcommand{\scrP}{\mathscr{P}}

\newcommand{\sfP}{\mathsf{P}}
\newcommand{\sfQ}{\mathsf{Q}}

\newtheorem{theorem}{Theorem}[section] 
\newtheorem*{theorem*}{Theorem} 

\newtheorem{lemma}{Lemma}[section]

\begin{document}

\maketitle

\begin{abstract} 
In this note, we give a short information-theoretic proof of the consistency of the Gaussian maximum likelihood estimator in linear auto-regressive models. Our proof yields nearly optimal non-asymptotic rates for parameter recovery and works without any invocation of stability in the case of finite hypothesis classes.
\end{abstract}

\section{Introduction}

Learning the dynamics of a linear dynamical system is a classical problem in for instance signal processing, system identification and econometrics. It is also arguably one of the simplest examples of an auto-regressive learning problem, thereby rendering it an instance of self-supervised learning. Understanding the sample complexity---and which quantities are of relevance for it---of such learning problems is key in the current era of large language models. 

The traditional approach for analyzing sequential (self-) supervised learning problems operates via comparison of the empirical and population excess risk functionals. Recently, \citet{jeon2024information} provided an information-theoretic proof approach for learning from dependent data eschewing any such direct comparison. However, their results only apply to the Bayesian setting. Nevertheless, this eschewing  of reasoning of the lower tail of the empirical risk is highly desirable as directly proving (anti-)concentration inequalities relating these quantities typically comes with significant technical overhead even in the linear setting \citep{simchowitz2018learning}. The situation is even worse in the nonlinear setting, where most known extensions require various mixing (stochastic stability) notions that seem excessive \citep{ziemann2024sharp}. Moreover, there is no reason to believe that many relevant time-series applications, such as natural language, are  mixing stochastic processes.

Inspired by the recent   information-theoretic Bayesian analysis of \cite{jeon2024information} we point out that tools from information theory can also be used in the frequentist setting to establish parameter recovery bounds for linear system identification. Moreover, the advantages of these ideas, notably sidestepping control of the lower tail, also extend. We present our illustration of this below.

\begin{theorem}
   
    Fix $W_{1:n} \sim N(0,I)$ and let $\sfP_{A_\star}$ be such that the $Z_{1:n}$ satisfy $Z_{k}=A_\star Z_{k-1}+W_k$ for $k =2,\dots n$ and $Z_1=W_1$. The maximum likelihood estimator $\widehat A$ (defined in \Cref{sec:hellingerlearning}) over any hypothesis class of the form $\scrP=\{\sfP_{A} : A \textnormal{ varies}\}$ containing $\sfP_\star=\sfP_{A_\star}$ achieves:
    \begin{equation}\label{eq:mainthmeq}
         \E \tr\left(\left(A_\star-\widehat A\right)^\T\left(A_\star-\widehat A\right) \frac{1}{n}\sum_{i=1}^n \sum_{k=1}^{i}   A_\star^{k-1}A_\star^{\T,k-1}\right) \leq (2 \times 10^4 )\times   \frac{I(\widehat \sfP\parallel Z_{1:n})}{n}
    \end{equation}

\end{theorem}
We remark that 1) $I(\widehat \sfP\parallel Z_{1:n}) \leq \log |\scrP|$ for any finite hypothesis class $\scrP$ and so the right hand side of \eqref{eq:mainthmeq} admits control decaying with $n$---the estimator $\widehat A$ is consistent and converges at a so-called fast rate; 2) whenever $A_\star$ is sufficiently stable, the result can be extended to parametric hypothesis classes isometric to compact subsets of Euclidean space via a standard discretization argument; 3) the rather large constant $(2 \times 10^4 )$ is a consequence of instantiating a result of \citet[Theorem 1.1]{devroye2018total} and there has been no attempt in the literature to optimize the corresponding constant in their result; and 4) while we have side-stepped control of the lower tail of the empirical risk functional, we have instead relied on the approximately closed form of the Gaussian total variation distance---a luxury we do not have for general learning problems. However, one might instead hope to exploit the fact that $f$-divergences in parametric families are locally a quadratic.

\section{Information-Theoretic Preliminaries}

For two probability measures $\sfP,\sfQ$ defined on the same probability space we denote their KL-divergence $d_{\mathrm{KL}}(\sfP \parallel \sfQ) \triangleq \int  \log \frac{\mathrm{d}\sfP}{\mathrm{d}\sfQ}\mathrm{d}\sfP$, their total variation distance by $d_{\mathrm{TV}}(\sfP \parallel \sfQ) \triangleq \frac{1}{2} \int \left |\frac{\mathrm{d}\sfP}{\mathrm{d}\lambda}- \frac{\mathrm{d\sfQ}}{\mathrm{d}\lambda}  \right| d\lambda $ and their squared Hellinger distance by $d_{\mathrm{H}}^2(\sfP \parallel \sfQ) \triangleq \frac{1}{2}\int \left(\sqrt{\frac{\mathrm{d}\sfP}{\mathrm{d}\lambda}}- \sqrt{\frac{\mathrm{d\sfQ}}{\mathrm{d}\lambda}} \right)^2\mathrm{d} \lambda$ where $\lambda$ is a joint dominating measure. If $(X,Y) \sim \sfP_{X,Y}$ we denote their mutual information by $I(X \parallel Y) \triangleq d_{\mathrm{KL}}(\sfP_{X,Y} \parallel \sfP_{X} \otimes \sfP_Y) $.

\begin{lemma}[\cite{donsker1975asymptotic}]
Fix two probability measures $\mathsf{P}$ and $\mathsf{Q}$ on a common measure space $(\Omega,\mathcal{F})$ with $\mathsf{P}\ll \mathsf{Q}$. Then:
    \begin{equation}
        d_{\mathrm{KL}}(\sfP \parallel \sfQ) = \sup_F \left\{\int_\Omega  F(\omega) \mathrm{d}\sfP(\omega) - \log \int_{\Omega}  e^{F(\omega)}\mathrm{d}\sfQ(\omega) \right\}
    \end{equation}
    where the supremum is taken over $\mathcal{F}$-measurable  and $\mathsf{P}$-exponentially-integrable $F:\Omega \to \mathbb{R}$.
\end{lemma}

\begin{lemma}\label{lem:hellingerform}
Fix two probability measures $\sfP$ and $\sfQ$ and let $\lambda$ be a joint dominating measure. We have that:
    \begin{equation}
        d_{\mathrm{H}}^2(\sfP \parallel \sfQ) =1-\int \exp \left(\frac{1}{2}\log \frac{\mathrm{d\sfQ}}{\mathrm{d}\lambda}-\frac{1}{2}\log {\frac{\mathrm{d}\sfP}{\mathrm{d}\lambda} } \right) \mathrm{d}\sfP.
    \end{equation}
\end{lemma}

\begin{proof}
Elementary algebraic manipulations give us that:
    \begin{equation}
        \begin{aligned}
            d_{\mathrm{H}}^2(\sfP \parallel \sfQ) &=\frac{1}{2} \int \left(\sqrt{\frac{\mathrm{d}\sfP}{\mathrm{d}\lambda} }- \sqrt{\frac{\mathrm{d\sfQ}}{\mathrm{d}\lambda}} \right)^2\mathrm{d} \lambda
            =  1-\int \sqrt{\frac{\mathrm{d}\sfP}{\mathrm{d}\lambda} \frac{\mathrm{d\sfQ}}{\mathrm{d}\lambda}} \mathrm{d} \lambda
            \\
             &=  1-\int \sqrt{\frac{\frac{\mathrm{d\sfQ}}{\mathrm{d}\lambda}}{\frac{\mathrm{d}\sfP}{\mathrm{d}\lambda} }} \mathrm{d}\sfP
             = 1-\int \exp \left(\frac{1}{2}\log \frac{\mathrm{d\sfQ}}{\mathrm{d}\lambda}-\frac{1}{2}\log {\frac{\mathrm{d}\sfP}{\mathrm{d}\lambda} } \right) \mathrm{d}\sfP
        \end{aligned}
    \end{equation}
    where we note that the integral is $0$ if $\frac{\mathrm{d}\sfP}{\mathrm{d}\lambda} =0$ justifying the fraction following the third equality.
\end{proof}

\section{Learning Generative Models in Hellinger Distance}\label{sec:hellingerlearning}
Let $Z_{1:n}$ be a sequence drawn according to $\sfP_\star$ and let $\widehat \sfP$ be the maximum likelihood estimator (MLE) over a class $\scrP$ with joint dominating measure $\lambda$. Recall that a distribution $\widehat \sfP$ is a MLE over $\scrP$ if $ -\log \frac{\mathrm{d}\widehat \sfP}{\mathrm{d}\lambda}(Z_{1:n}) + \log \frac{\mathrm{d}\sfP}{\mathrm{d}\lambda}(Z_{1:n}) \leq 0,  \forall \sfP \in \scrP$. We begin by a variation of the information-theoretic analysis of \cite{zhang2006varepsilon} inspired by the mutual-information decoupling of \cite{xu2017information}.


\begin{theorem}
\label{thm:zhangthm}
Let $\widehat \sfP$ be the maximum likelihood estimator over $\scrP$.
If $\sfP_\star \in \scrP$ we have that:

    \begin{equation}\label{eq:zhangthmeq}
       \E \left[ d_{\mathrm{H}}^2(\widehat \sfP \parallel \sfP_\star )\right]
        \leq -  2\log \left( 1-\frac{1}{2}  \E_{\widehat \sfP } d_{\mathrm{H}}^2(\widehat \sfP \parallel \sfP_\star )\right)\leq 2 I(\widehat \sfP\parallel Z_{1:n}).
    \end{equation}
\end{theorem}

The key observation is that the right hand side of \eqref{eq:zhangthmeq} does not necessarily grow with $n$, whereas the left measures the distance between distributions of an increasing number of variables.

\begin{proof}
        %
        %

Let $\lambda$ be a joint dominating measure for $\sfP_\star$ and $\widehat \sfP$ (e.g. their mixture). We have that
\begin{equation}\label{eq:hellingertaylor}
    \begin{aligned}
        -&\log \E_{Z\otimes \widehat \sfP } \exp \left( \frac{1}{2} \log \frac{\mathrm{d}\widehat \sfP}{\mathrm{d}\lambda}-\frac{1}{2}\log \frac{\mathrm{d}\sfP_\star}{\mathrm{d}\lambda} \right) 
        \\
        &
         = -  
         \log \E_{\widehat \sfP } \E_{Z}\exp \left( \frac{1}{2} \log \frac{\mathrm{d}\widehat \sfP}{\mathrm{d}\lambda}-\frac{1}{2}\log \frac{\mathrm{d}\sfP_\star}{\mathrm{d}\lambda} \right) 
         &&(\textnormal{Fubini})
         \\
         &
         = -  \log \left( 1-\frac{1}{2}  \E_{\widehat \sfP } d_{\mathrm{H}}^2(\widehat \sfP \parallel \sfP_\star )\right)
         &&(Z\sim \sfP_\star \textnormal{ and \Cref{lem:hellingerform}})
         \\
         &
         \geq \frac{1}{2}
         \E_{\widehat \sfP } d_{\mathrm{H}}^2(\widehat \sfP \parallel \sfP_\star ). &&(x\in [0,1)] \Rightarrow x \leq - \log (1-x))
    \end{aligned}
\end{equation}
To finish we instantiate Donsker-Varadhan with $\sfQ = \sfP_{Z} \otimes\sfP_{ \widehat \sfP }$ and $\sfP = \sfP_{Z ,\widehat \sfP}$:
\begin{equation}\label{eq:hellingerDV}
\begin{aligned}
    -\log \E_{Z\otimes \widehat \sfP } \exp \left( \frac{1}{2} \log \frac{\mathrm{d}\widehat \sfP}{\mathrm{d}\lambda}-\frac{1}{2}\log \frac{\mathrm{d}\sfP_\star}{\mathrm{d}\lambda} \right) &
    \leq -\frac{1}{2} \E_{Z ,\widehat \sfP} \log \frac{\mathrm{d}\widehat \sfP}{\mathrm{d}\sfP_\star} + I(\widehat \sfP \parallel Z_{1:n})
    &&(\textnormal{Donsker-Varadhan})
    \\
    &
 \leq I(\widehat \sfP \parallel Z_{1:n}). &&(\textnormal{Optimality of MLE})
    \end{aligned}
\end{equation}
By combining \eqref{eq:hellingertaylor} with \eqref{eq:hellingerDV} the result follows.
\end{proof}


\section{Proof of the Main Result}


    
    Let us write that $\sfP_A$ for the distribution under which $Z_{1:n}=\mathbf{L}_A W_{1:n}$ where
    \begin{equation}
    \mathbf{L}_A=
        \begin{bmatrix}
            I & 0 &0 &\cdots & 0 \\
            A & I & 0 &\cdots & 0\\
            A^2 & A & I &\ddots &\vdots\\
            \vdots &\ddots &\ddots &\ddots &\vdots\\
            A^{n-1} & A^{n-2} & \cdots & \cdots & I
        \end{bmatrix}
    \quad
    \textnormal{and}
    \quad
    \mathbf{L}_A^{-1}=        \begin{bmatrix}
            I & 0 &0 &\cdots & 0 \\
            -A & I & 0 &\cdots & 0\\
            0 & -A & I &\ddots &\vdots\\
            \vdots &\ddots &\ddots &\ddots &\vdots\\
            0& 0 & \cdots & -A & I
        \end{bmatrix}.
        \end{equation}
Let $\Sigma_A = \mathbf{L}_A \mathbf{L}_A^\T $ be the covariance matrix of the joint distribution under $\sfP_A$. We define $\mathbf{L}_{A_\star}$ and $\Sigma_{A_\star}$  analogously.  By \citet[Theorem 1.1]{devroye2018total} is suffices to control $\tr  \left( \Sigma_{A_\star}\Sigma_A^{-1}- I\right)$. Namely, we have that (where the second inequality follows by Cauchy-Schwarz):
    \begin{equation}
        \tr   \left( \Sigma_{A_\star}\Sigma_A^ {-1}- I\right) \leq 10^4 d_{\mathrm{TV}}^2(\sfP_A \parallel \sfP_\star ) \leq 10^4 d_{\mathrm{H}}^2(\sfP_A \parallel \sfP_\star ).
    \end{equation}
 and use the trace cyclic property to observe that 
\begin{equation}
    \tr \Sigma_{A_\star}\Sigma_A^ {-1} = \tr  \mathbf{L}_{A_\star}\mathbf{L}_{A_\star}^\T(\mathbf{L}_{A}\mathbf{L}_{A}^\T)^{-1} =\tr  \mathbf{L}_{A}^{-1}\mathbf{L}_{A_\star}\mathbf{L}_{A_\star}^\T \mathbf{L}_{A}^{\T,-1} .
\end{equation} 
Straightforward calculation now yields that
\begin{equation}
    \mathbf{L}_{A}^{-1}\mathbf{L}_{A_\star}  = \begin{bmatrix}
        I & 0 & 0 & 0 &\cdots&0\\
        A_\star-A & I & 0 &0&\cdots &0\\
        A^2_\star-A A_\star  & A_\star-A & I & 0 &\cdots&0\\
        \vdots &\ddots &\ddots &\ddots &\ddots &\vdots \\
        A^{n-1}_\star-A A_\star^{n-2} & A^{n-2}_\star-A A^{n-3}_\star  &\cdots & \cdots & A_\star -A & I
    \end{bmatrix}.
\end{equation}
And hence we may write for the diagonal elements
\begin{equation}
    [\mathbf{L}_{A_\star} \mathbf{L}_{A}^{-1} (\mathbf{L}_{A_\star} \mathbf{L}_{A}^{-1})^\T]_{ii} = I+\sum_{k=1}^{i} (A_\star^{k}- A A_\star^{k-1})((A_\star^{k}-A A_\star^{k-1}))^\T.
\end{equation}
In other words, by repeated use of the trace cyclic property:
\begin{equation}
\begin{aligned}
    \tr \left(\Sigma_{A}^{-1}\Sigma_{A_\star} - I \right) 
    &=
    \tr \left(\mathbf{L}_{A}^{-1} \mathbf{L}_{A_\star} ( \mathbf{L}_{A}^{-1}\mathbf{L}_{A_\star})^\T- I\right)\\
    &=
    \tr \sum_{i=1}^n \sum_{k=1}^{i} (A_\star^{k}-A A_\star^{k-1})(A_\star^{k}- A A_\star^{k-1})^\T\\
    &=
     \tr \left((A_\star-A)^\T(A_\star-A) \sum_{i=1}^n \sum_{k=1}^{i}   A_\star^{k-1} A_\star^{\T,k-1} \right).
\end{aligned}
\end{equation}
The result follows by instantiating the above with $\sfP_A=\sfP_{\widehat A}$. \hfill $\blacksquare$



\addcontentsline{toc}{section}{References}

\bibliographystyle{plainnat}

\bibliography{main.bib}

\end{document}